\documentclass{article}

\usepackage{subcaption}
\usepackage{graphicx}
\usepackage{float}
\usepackage{placeins}
\usepackage{amsmath}
\usepackage{booktabs}
\usepackage{multirow} 
\usepackage{theorem}
\usepackage{tikz}
\usepackage{tikz-cd}
\usepackage[dvipsnames]{xcolor}
\usetikzlibrary{shapes.geometric, arrows, positioning, arrows.meta, calc, backgrounds, matrix, shapes.symbols, shapes.misc, decorations.markings, shadows, decorations.pathreplacing, backgrounds}
\pgfdeclarelayer{background}
\pgfsetlayers{background,main}

\usepackage{tcolorbox}
\definecolor{DodgerUniformBlue}{rgb}{0.0,0.353,0.612}
\newcommand{\define}[1]{\emph{\textcolor{DodgerUniformBlue}{#1}}}

\newcommand{\openbox}{\leavevmode
  \hbox to.77778em{%
  \hfil\vrule
  \vbox to.675em{\hrule width.6em\vfil\hrule}%
  \vrule\hfil}}

\def\endupshapeenvironment{%
  \ifvmode\else\hfill\quad\fi
  \hbox{\openbox}}
\def\qed{\hfill\openbox}

\providecommand{\mytheoremheaderfont}{\upshape\bfseries}
\theoremheaderfont{\mytheoremheaderfont}

\theoremstyle{plain}
\theorembodyfont{\slshape}
\newtheorem{theorem}{Theorem}[section]

\theorembodyfont{\upshape}
\newtheorem{rmdefinition}[theorem]{Definition}
\newtheorem{rmprocedure}[theorem]{Procedure}

\newtheorem{rmexample}[theorem]{Example}
\newtheorem{rmnote}[theorem]{Note}
\newtheorem{rmconvention}[theorem]{Convention}
\newtheorem{rmalgorithm}[theorem]{Algorithm}
\newtheorem{rmremark}[theorem]{Remark}

\newenvironment{algorithm}{\begin{rmalgorithm}}{\endupshapeenvironment\end{rmalgorithm}}

\newenvironment{proof}
        {\begin{trivlist} \item[] \textsl{Proof.}}
        {\qed \end{trivlist}}

\newcommand{\repeattheorem}[2]{\renewcommand{#2}{\protect\ref{#1}}}
\newcommand{\starversion}[3]{
\newenvironment{#2}[2][\foo]%
   {\repeattheorem{##2}{#3}%
     \def\foo{void}\ifx##1\foo%
            \begin{#1}%
            \else\begin{#1}[##1]\fi}%
   {\end{#1}}}

\starversion{theorem}{theorem*}{\thetheorem}
\starversion{lemma}{lemma*}{\thetheorem}
\starversion{proposition}{proposition*}{\thetheorem}
\starversion{corollary}{corollary*}{\thetheorem}

\usepackage{algorithm} 
\usepackage{algorithmic}

\usepackage[position, final]{neurips_2026}

\usepackage[utf8]{inputenc} 
\usepackage[T1]{fontenc}    
\usepackage{hyperref}       
\usepackage{url}            
\usepackage{booktabs}       
\usepackage{amsfonts}       
\usepackage{nicefrac}       
\usepackage{microtype}      
\usepackage{xcolor}         

\title{Beyond State-of-the-Art: Standardising Environmental Impact Metrics for AI Research}

\author{%
 Lachlan McGinness \\
 School of Computing, Australian National University\\
  Commonwealth Scientific and Industrial Research Organisation \\
  \texttt{lachlan.mcginness@anu.edu.au} \\
  \And
  Dan Pagendam \\
  Commonwealth Scientific and Industrial Research Organisation \\
  \AND
  Robert Offner \\
  School of Computing, Australian National University\\
}

\begin{document}

\maketitle

\begin{abstract}
  As the capabilities and ubiquity of Large Language Models (LLMs) grow, so does their environmental footprint. 
  Despite calls for responsible AI, the machine learning community lacks standardised practices for carbon accounting. 
  Our automated literature review of the 5,285 papers accepted to NeurIPS 2025 reveals that reporting of environmental impact is nearly non-existent. 
  To catalyse a shift toward sustainable AI, we define standardised sustainability metrics for evaluating model training efficiency, accompanied by simple heuristics to estimate the carbon cost of LLM inference. 
  We implement these metrics in \texttt{carbonbenchmark}, a drop-in software solution for tracking and reporting emissions. 
  Finally, to combat the pursuit of marginal accuracy gains at disproportionate environmental costs, we formalise the `Smallest Model that Achieves the Job' (SMAJ), a framework which challenges the field to prioritise computational efficiency and environmental accountability alongside traditional `State-of-the-Art' (SotA) accuracy.
\end{abstract}

Total Carbon Emissions of this Study: $2.49$kg of CO$_2$

\section{Introduction}

In 2026, AI headlines such as `An engineer at OpenAI processed 210 billion tokens in one week' or `At Meta AI total employee usage exceeded 60 trillion tokens' have become commonplace. 
But there is an invisible cost to the enormous number of tokens generated by Large Language Models (LLMs): carbon dioxide.
But exactly how much CO$_2$ is 210 billion tokens? 
How much energy are Machine Learning (ML) and LLMs using?
The first step in allowing people to determine whether this is acceptable or not is to quantify and accurately report the carbon footprint of AI models.

AI and machine learning research has embraced models so large that they are now at risk of becoming a serious contributor to anthropogenic global warming \cite{strubell2019energy}.  There is much talk of \emph{responsible AI} in research and industry (see for example \cite{hartman2025rad}), and of fundamental importance in this conversation is managing its ecological footprint, referred to as \emph{sustainable AI}\cite{vanwynsberghe2021sustainable}.  We believe that our research field needs to overcome three important issues: 
(i) foundation models need to become transparent in providing an individual with reliable estimates of their greenhouse gas emissions (in tonnes of CO$_2$ equivalent); 
(ii) reporting of contributions to greenhouse gas emissions from research needs to become a standard metric reported in research papers; and 
(iii) we need increased awareness that research advances are not justified at any cost to the environment.  We argue that tackling these three challenges will help to drive AI research towards computational efficiency, thoughtful usage, and ``greener'' hardware/infrastructure.

The first point we raise asks for greater transparency from the technology industry (a significant player in this research field).  
We propose that users have a right to understanding the environmental footprint of the AI tools that they use and to balance the choice of foundation model against these impacts.  Such transparency could be a driver of sustainability across the AI industry, and could be a key driver of AI infrastructure adopting sustainable electricity sources, and embracing a mindset that balances industrial competition with accountability and responsibility.

Our second point speaks to the need for reporting of carbon accounting and sustainability metrics for research conducted in the present.  
This is important if we are to be honest and accountable for our impact on the environment.  
When (i) is achieved, researchers, conferences, journals, and funding bodies are empowered to report carbon footprints associated with our science, and to weigh these against societal benefits.

We argue that as a rapidly growing research field, AI has the power to do much good, but it should embrace environmental accountability if it is to truly embrace being responsible and sustainable in this endeavor.  To this end, we need to increase researcher awareness around the ecological footprint that their research leaves and for them to exercise caution and conservatism when conducting studies.  In the same way that researchers in medical research use statistical power calculations to reduce the number of subjects in a clinical trial, AI researchers should be encouraged to make choices that reduce the mass of carbon that their research emits to the atmosphere whilst still achieving desirable research outcomes.

Although environmental transparency for many commercially available LLMs is currently lacking, this paper provides simple methods that researchers can adopt to estimate the carbon footprints from their studies.  
This enables researchers to be proactive about adopting principles of \emph{sustainable AI} and encouraging the AI industry to meet our need for greater environmental transparency.  
In addition, we advocate for the reporting of carbon emissions of all research; and introduce a training sustainability metric ($k$) that can provide a useful way to assess and compare the carbon efficiency of new ML methods that learn iteratively.  
This is a step towards finding a metric that can be used by downstream users of AI technologies to assess the computational tradeoff between accuracy and carbon footprint.  

We present both theoretical and empirical results using well-established models and benchmarks to demonstrate how $k$ can be used to determine training efficiency. 
A further contribution of this paper is in defining the concept of `Smallest Model that Achieves the Job' (SMAJ) as an alternative concept to `State-of-the-Art' (SotA; a term that encourages marginal gains in accuracy without regards to computational cost), and demonstrating how small pilot studies and statistical methods can be used to determine SMAJ among a set of candidate models.

\section{LLM Literature Review}\label{section:LLM_Lit_Review}

To better understand what portion of the NeurIPS community use LLMs, measure computational expense, and report environmental impact, we conducted an automated literature review of the 5285 papers accepted to the main track of NeurIPS 2025. 
We extracted relevant information from each of the papers in an automated way using LLMs. 
First we chose a random selection of 200 papers and tested eight different LLMs of different sizes on these. We searched for the following markers:
\begin{enumerate}
    \item \textbf{Did the paper use LLMs? (Yes/No):} An indicator of whether the paper reported using LLMs as a part of the study.
    \item \textbf{Types of LLM Usage:} In what ways were researchers using LLMs in their studies? There were five categories and each paper could respond yes to multiple categories: 
    \begin{itemize}
        \item \textbf{Architecture/Pre-training (Yes/No)} - if the paper proposes a new model
        \item \textbf{Fine-tuning (Yes/No)} - if the paper adjusts an existing model for a specific task
        \item \textbf{Inference/Evaluation (Yes/No)} - if the paper uses the LLM as a judge for labeling data
        \item \textbf{Synthetic Data Generation (Yes/No)} - if the paper uses LLMs to create training data
        \item \textbf{Agentic Frameworks (Yes/No)} - if the LLM is used for reasoning, solving a problem or completing a task
    \end{itemize}
    \item \textbf{Were the LLMs Open Source (Yes/No):} If the LLMs used were open source?
    \item \textbf{Computational Cost (Yes/No):} To determine if the paper reports the computational cost of its methods (even if the paper is not about LLMs). 
    \item \textbf{Sustainability (Yes/No):} To determine whether the carbon footprint or any other environmental impacts were mentioned in the paper.
\end{enumerate}
We calculate the percentage agreement between each of the eight LLMs across all of the papers and markers. The results are shown in the Figure \ref{fig:agreement_heatmap} heatmap and further discussed in Appendices \ref{appendix:literature_review} and \ref{section:appendix_statistics}.
\begin{figure}[ht]
    \centering
    \includegraphics[width=0.7\textwidth]{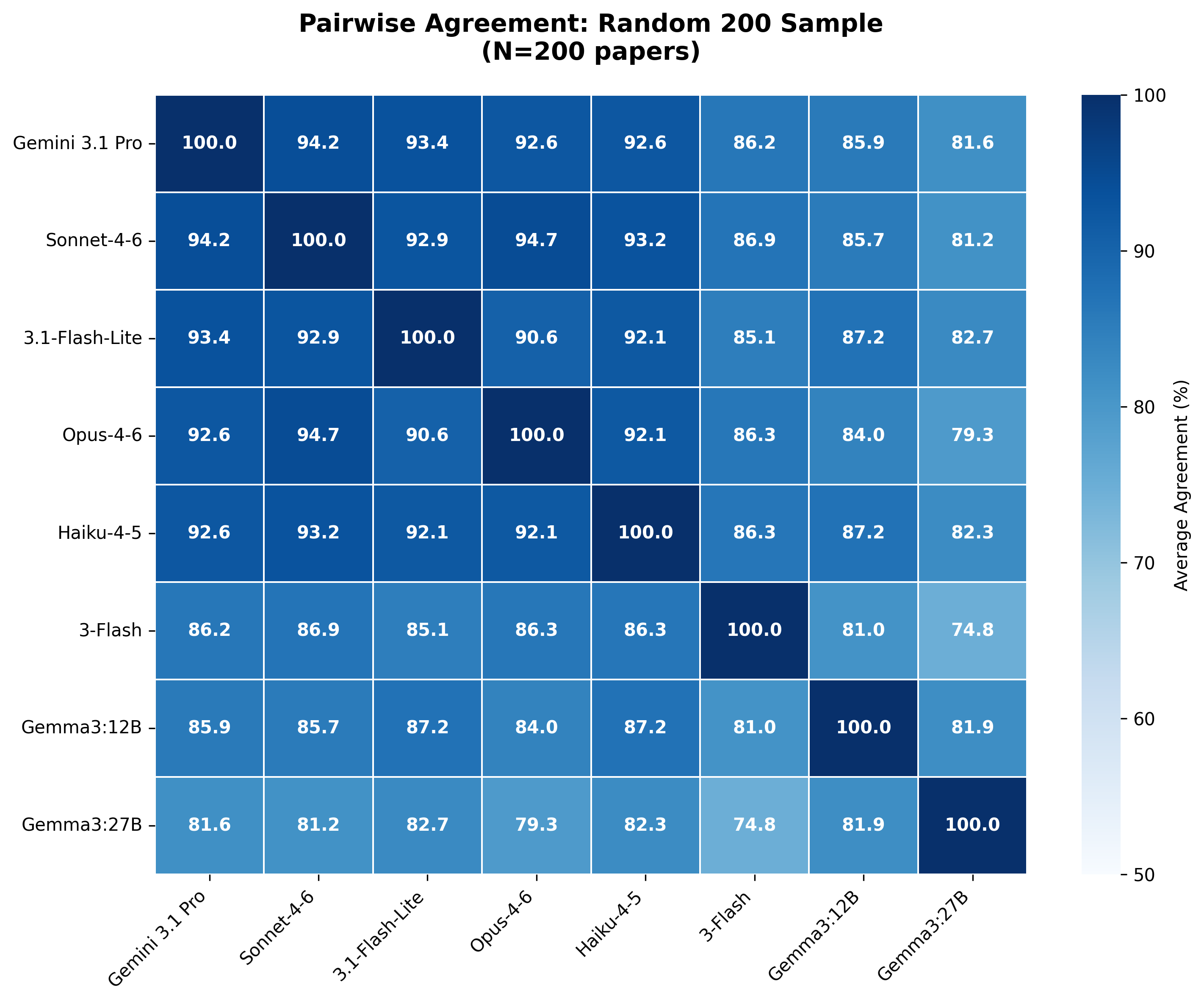}
    \caption{Heatmap showing agreement levels between each of the models on extracted data from 200 randomly selected papers. 
    There is strong agreement ($90\%+$) between the five strongest models.
    Therefore in order to reduce computational cost (and therefore environmental impact), the smallest of these five models (Gemini 3.1-Flash-Lite) was chosen to conduct the literature review.}
    \label{fig:agreement_heatmap}
\end{figure}

In Appendix \ref{section:appendix_statistics} we undertake a statistical analysis, whereby each LLM is compared to a State of the Art (SotA) baseline model (here chosen to be Gemini 3.1 Pro).  
We conclude that there is no statistically significant difference between Gemini 3.1 Pro and each of Sonnet-4.6, Gemini 3.1-Flash-Lite, Opus-4.6, and Haiku-4.5.  
Using the philosophy of Smallest Model that Achieves the Job  ({SMAJ} see Section \ref{sec:SMAJ}) we chose to use Gemini 3.1 Flash Lite to review the remaining papers.

The review found that approximately 40\% of papers accepted for NeurIPS 2025 used LLMs in their methods with the most common uses being Fine Tuning (22\% of papers), Inference and Evaluation (19\%), and Agentic Frameworks (17\%). 
More than half the the papers which used LLMs used open-source models (24\% of total papers).
Approximately 40\% of papers reported the computational cost of their methods in some way, regardless of whether they used LLMs or not.
But almost no papers (approximately 1\%) mention sustainability or considered the environmental impacts of their methods. 
In the small number of papers where sustainability is considered, there is no standard method for reporting computational cost, power usage or CO$_2$ released. 

We propose a submission requirement that every paper reports the total mass of carbon dioxide emitted by its methods and experiments. 
We believe that just as the act of recording what you eat can cause people to improve their diet, requiring researchers to publish their CO$_2$ usage will result in more environmentally friendly research.
In the following section we outline more specific measures that could be used to quantify the sustainability of training ML models iteratively.

\section{Theoretical Underpinnings of Sustainability Benchmarks}

Previous works have suggested that metrics that present the CO$_2$ emissions per percentage point of accuracy could be a starting point for determining the efficiency or sustainability of a ML method \cite{Eimler2026Environmental}. As these methods are likely to reward under-trained models, an alternate metric has been proposed \cite{Eimler2026Environmental}:
\begin{equation}
    \label{eq:AIED}
    k_{\text{sustainability}} = \frac{\log_e(1-A)}{t}
\end{equation}
Where $k_{\text{sustainability}}$ is a constant representing the sustainability of training, $t$ is a measure of computational expense (for example FLOPs, Joules or grams of CO$_2$) and $A$ is the maximum accuracy achieved in training.
This metric has some limitations: firstly the formula breaks when an accuracy $A=1$ is achieved. 
Secondly, it was derived assuming that the learning curve  takes the form $A=1-e^{-k_\text{sustiainability}t}$ and will likely reward under-trained models when the highest accuracy that can be achieved by any method is significantly less than 1.
Finally it assumes that the accuracy takes an exponential form achieved by Sorscher \cite{Sorscher2022Beyond} rather than the almost universally-assumed, less efficient, power law form \cite{kaplan_scaling_2020,Hoffmann2022Training} given by Equation \ref{eq:powerlaw}. 

The first two concerns can be addressed by assuming that the gap between achieved accuracy and maximum potential accuracy decays exponentially as a function of $t$, so that the learning curve takes the form: 
\begin{equation}
\label{eq:exponential}
 A=A_{\text{max}}(1-e^{-kt})
\end{equation}
This equation has two parameters which must be fitted from the data: $A_{\text{max}}$ (related to accuracy) and $k$ (related to training sustainability). 
Unfortunately this equation cannot be solved analytically for $k$ to produce an equation like \ref{eq:AIED}.

\begin{theorem}
\label{theorem:exponential}
    No general closed-form, analytic solution for training efficiency parameter $k$ can be determined from known values of $A$ and $t$, if we assume the learning curve takes the form $ A=A_{\text{max}}(1-e^{-kt})$ (Equation \ref{eq:exponential}). 
\end{theorem}

For proof of Theorem \ref{theorem:exponential}, see Appendix \ref{Appendix:proof_of_theorems}.
As we are interested in finding a method that will allow the general reporting of training efficiency, we instead recommend that $k$ and $A_{max}$ can be obtained by least squares fitting to the learning curve as demonstrated in Figure \ref{fig:vision_vit}.

\begin{figure}[htbp]
    \centering
    \includegraphics[width=\textwidth]{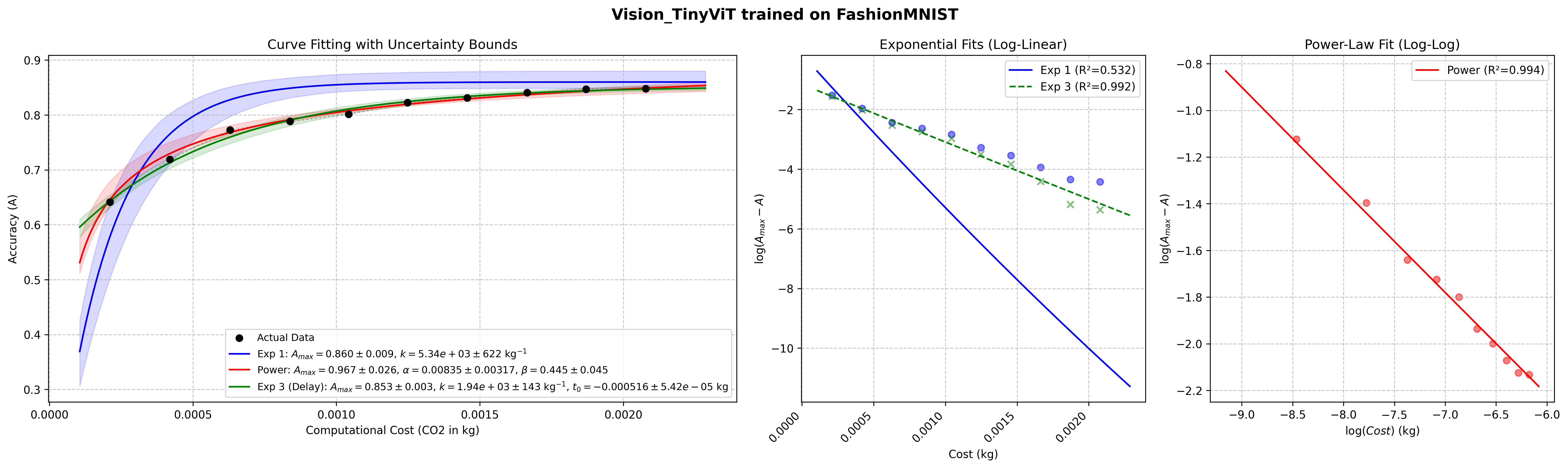}
    \caption{Learning curve and mathematical fits evaluating the computational efficiency of TinyViT trained on the FashionMNIST dataset. Blue curves show the simple exponential model of training accuracy, green shows the delayed exponential model, and red shows the more accurate power-law model.  This demonstrates that both the power law and delayed exponential functions can accurately fit the learning curve, allowing $k$ and $\beta$ values to be used as measures of training efficiency. }
    \label{fig:vision_vit}
\end{figure}

Figure \ref{fig:vision_vit} also demonstrates that the simple exponential function (depicted in blue) from Equation \ref{eq:exponential} does not fit the learning curve well. That is because the functional form forces the curve to pass through $(0,0)$ which in practice means that $A_{\text{max}}$ is largely setting the fit for the final points and $k$ is very heavily influenced by the first data point. In order to repair this and allow for better fitting of the exponential function, we recommend allowing for a translation of $t$ by $t_0$ which relaxes the requirement of the curve passing through $(0,0)$, as shown in Equation \ref{eq:delayed_exponential}.

\begin{equation}
\label{eq:delayed_exponential}
 A=A_{\text{max}}(1-e^{-k(t-t_0)})
\end{equation}

We refer to this model as the delayed exponential model.  
Equations \ref{eq:exponential} and \ref{eq:delayed_exponential} are simple ways to determine training efficiency as there is still just a single parameter ($k$) which describes the sustainability of training.
By contrast, the more commonly used power law (Equation \ref{eq:powerlaw}) has three parameters $A_{\text{max}}$, $\alpha$, and $\beta$. 
\begin{equation}
\label{eq:powerlaw}
  A = A_{\text{max}}(1 - \alpha t^{-\beta}) 
\end{equation}
If we assume that $A_{\text{max}}$ is close to $1$, then Equation \ref{eq:powerlaw} recovers the normal form expressed in error rate which is most regularly reported \cite{kaplan_scaling_2020,Hoffmann2022Training} as  
\begin{equation*}
    E(t) \approx 1 - A = \alpha t^{-\beta}
\end{equation*}
Where $E(t)$ is the error rate as a function of training. 
The assumption of error rate approaching zero or (accuracy approaching 1) leads to the regime of irreducible error where power law scaling has been observed to break down. 
By explicitly introducing an irreducible error rate, Khorasani et. al. demonstrate that the power law for reducible loss continues even when total loss plateaus \cite{Khorasani2026Beyond}.
This indicates that $A_{\text{max}}$ (or irreducible loss) is required to accurately model power law scaling. 

For Equation \ref{eq:powerlaw}, $A_{\text{max}}$ is still an estimate of the highest accuracy which is achieved after infinite training.
$\alpha$ is a measure of position of the model in early training, specifically how far the model is away from $A_{\text{max}}$ after expending $t=1$ training resources and therefore has units of inverse $t^\beta$.
As the units of $\alpha$ depend on $\beta$, this makes it much more difficult to interpret and compare with $\alpha$ values for other models.
By contrast $\beta$ is dimensionless and indicates how quickly $A_{max}$ will be approached in the limits of large amounts of training. 
Therefore we debate that both $\alpha$ and $\beta$ are important parameters when reporting training efficiency.
Similar to the exponential fit to the training curve, it is not possible to analytically solve for  $\alpha$ and $\beta$ from training data.

\begin{theorem}
\label{theorem:power_law}
    No general closed-form, analytic solution for training efficiency parameter $\beta$ can be determined from known values of $A$ and $t$, if we assume the learning curve takes the form $ A = A_{\text{max}}(1 - \alpha t^{-\beta})$ (Equation \ref{eq:powerlaw}). 
\end{theorem}

For proof of Theorem \ref{theorem:power_law} see Appendix \ref{Appendix:proof_of_theorems}.
Any machine learning process which produces a training curve can use 
a least squares fit to find $\alpha$ and $\beta$, similar to $k$. 
Although it is likely that a power law curve will be a better fit to training data than an exponential function, pairs of $\alpha$ and $\beta$ values are not as easily compared or interpreted as the single number $k$. 
Therefore our recommendation is that reporting of carbon emissions for a research study is mandatory and we recommend the reporting of $k$-values as the standardised benchmark to report the carbon efficiency of an ML method. 
If researchers feel that their combination of $\alpha$ and $\beta$ better represents their represent their training efficiency they could choose to additionally report this.

\subsection{Numerically Fitting the Learning Curve}
\label{sec:numerical_fitting}

As it has been shown that the exponential and power law functions cannot be fit to learning curves analytically, we outline the well-established least squares procedure for achieving this numerically in Appendix \ref{appendix:least_squares}. 

We provide a repository \texttt{carbonbenchmark}\footnote{https://github.com/LudvikGalois/carbonbenchmark} which uses \texttt{codecarbon} to measure the CO$_2$ emissions before implementing our fitting procedure \cite{codecarbon}.
We also provide a notebook with examples of \texttt{carbonbenchmark} being used to train well-known models on established datasets.
We ran the experiments on a machine with an NVIDIA GeForce RTX 4060 Ti GPU (16GB of VRAM) over a one hour period, but GPUs with 8GB of VRAM are sufficient.
Our examples include tasks in the computer vision (Fashion-MNIST \cite{FMNIST}), NLP (GLUE sst2 \cite{GLUE_Benchmark}), and tabular (UCI adult \cite{UCIdataset}) domains. 
\texttt{carbonbenchmark} provides the user with the learned sustainability parameters in Equations \ref{eq:exponential}, \ref{eq:delayed_exponential}, and \ref{eq:powerlaw}, as well as the final accuracy, the total carbon emissions, and the $R^2$ value for each fit.
Figure \ref{fig:vision_vit} gives an example output showing the fitting  to the learning curve for a vision transformer acting on the FashionMNIST dataset (chosen because it is more challenging that the original MNIST benchmark). 
More examples can be found in Appendix \ref{appendix:additional_results}.

On each of the three tasks, the models with higher values of $k$ required less $CO_2$ for training. This indicates that $k$ is effective at quantifying the CO$_2$ required for training.
We also note that the $t_0$ parameter was required for the exponential function to fit the data. The power-law fit was the best fit to most of the plots, especially after accuracy plateaus. 

\subsection{Smallest Model that Achieves the Job}
\label{sec:SMAJ}

Now we define the \define{Smallest Model that Achieves the Job}, a concept that we propose as essential to minimising cost when deploying machine learning models. 

For a given task let $\mathbf{M}=\{M_1, M_2, ..., M_n\}$ be a set of candidate models.
We represent each model as a tuple $M_i=(A_i,\Delta A_i, C)$, where: 
\begin{itemize}
    \item $A_i$ is the model's expected performance (for example, mean accuracy in testing data, or $A_{\text{max}}$ fitted to the learning curve during training).
    \item $\Delta A_i$ is the uncertainty associated with the performance metric (for example the standard deviation in model's accuracy across testing data, or the uncertainty in $A_{\text{max}}$ during training).
    \item $C$ is the expected or mean computational cost to deploy the model (for example measured in grams $CO_2$, Joules or FLOPs).
\end{itemize}

We define the \define{Highest Performing Model} by Equation \ref{eq:Maximum_Performing} 
\begin{equation}
    \label{eq:Maximum_Performing}
    M_{\text{HP}} = \underset{M_i \in \mathbf{M}}{\operatorname{argmax}}(A_i)
\end{equation}
Recognising that marginal gains in accuracy often occur at dispoportionate computational costs, we propose the Smallest Model that Achieves the Job (SMAJ) as follows.
First we define a set of acceptable models in Equation \ref{eq:acceptable_models}
\begin{equation}
    \label{eq:acceptable_models}
    \mathbf{M}_{\text{acceptable}}=\{M_i \in \mathbf{M} | A_{\text{HP}}-A_i \leq {\sqrt{(\Delta A_{\text{HP}})^2+(\Delta A_i)^2}} \}
\end{equation}
Intuitively, these are models whose accuracy falls within the combined uncertainty bounds of the highest performing model. 
The \define{Smallest Model that Achieves the Job} is defined by Equation \ref{eq:SMAJ}
\begin{equation}
    \label{eq:SMAJ}
    M_{\text{SMAJ}} = \underset{M_i \in \mathbf{M}_{\text{acceptable}}}{\operatorname{argmin}(C_i)}
\end{equation}

This procedure (outlined in Algorithm \ref{alg:smaj}) selects the most computationally efficient model which is statistically indistinguishable from the highest performing model.  In Section \ref{section:LLM_Lit_Review} (and Appendix \ref{section:appendix_statistics}) we demonstrated how small pilot experiments could be used to assess the statistical differences in the performances of models and to identify the SMAJ.
\begin{algorithm}[htbp]
\caption{Selection of the Smallest Model that Achieves the Job (SMAJ)}
\label{alg:smaj}
\begin{algorithmic}[1]
\REQUIRE A set of $n$ evaluated models $\mathbf{M} = \{M_1, M_2, \dots, M_n\}$.
\REQUIRE For each model $M_i$, a tuple $(A_i, \Delta A_i, C_i)$ representing performance, uncertainty, and computational cost.
\ENSURE The SMAJ model $M_{\text{SMAJ}}$.
\STATE \textbf{Initialize:} $\mathbf{M}_{\text{acceptable}} \leftarrow \emptyset$
\STATE \COMMENT{Step 1: Identify the Highest Performing (HP) model}
\STATE $M_{\text{HP}} \leftarrow \text{argmax}_{M_i \in \mathbf{M}} (A_i)$
\STATE Let $A_{\text{HP}}$ and $\Delta A_{\text{HP}}$ be the performance and uncertainty of $M_{\text{HP}}$.
\STATE \COMMENT{Step 2: Filter for models within the statistical uncertainty bound}
\FOR{each $M_i \in \mathbf{M}$}
    \STATE $\text{Tolerance} \leftarrow \sqrt{(\Delta A_{\text{HP}})^2 + (\Delta A_i)^2}$
    \IF{$A_{\text{HP}} - A_i \leq \text{Tolerance}$}
        \STATE $\mathbf{M}_{\text{acceptable}} \leftarrow \mathbf{M}_{\text{acceptable}} \cup \{M_i\}$
    \ENDIF
\ENDFOR
\STATE \COMMENT{Step 3: Select the most computationally efficient model from the acceptable set}
\STATE $M_{\text{SMAJ}} \leftarrow \text{argmin}_{M_i \in \mathbf{M}_{\text{acceptable}}} (C_i)$
\RETURN $M_{\text{SMAJ}}$
\end{algorithmic}
\end{algorithm}

We hope that SMAJ may become as highly sought and commonly used as State of the Art (SotA) in future research studies.

\section{Cost of Running LLMs}

Kaplan et al provided Equation \ref{eq:simple_flops}, an approximate formula for the Floating Point Operations (FLOPs) used by an LLM \cite{kaplan_scaling_2020}. This formula has been verified to be accurate within 10\% of the actual cost in some experiments \cite{Eimler2026Environmental,McGinness2025Imitate}.
\begin{equation}
    \label{eq:simple_flops}
    \text{FLOPs} = 2nN
\end{equation}
Where $n$ is the number of active parameters for a model and $N$ is the sum of context and completion tokens. 
The utility of Equation \ref{eq:simple_flops} is that the user only requires two pieces of information to estimate the computational cost of running an LLM. 

The first is $N$, which is very easy for researchers using LLMs to obtain. 
For local models, (like Ollama), there are inbuilt options for the user to obtain the number of input and output tokens for each call. For commercial providers, including those of proprietary models, users are normally billed by the token through APIs and therefore have access to the number of input and output tokens.

The second parameter, $n$, is well known for most open source models. For example the two open source LLMs used in this study, Gemma 3 12B and Gemma 3 27B have 12 billion and 27 billion active parameters respectively. 
However almost no proprietary models release the number of active parameters, making the exact value of $n$ unobtainable. 
We call for companies that create proprietary LLMs to publish the value of $n$, but in the meantime build on the practical solution originally proposed by Eimler et al \cite{Eimler2026Environmental}.

At the time of writing, open source models like Deepseek-v4-pro \cite{DeepseekV4} are similar in benchmark performance to the best proprietary models. 
If we assume that this means they have a similar number of active parameters, then we can use this as an approximation of $n$ for proprietary models. 
Although this is not a precise method, it allows users to make an order of magnitude estimate for the carbon dioxide emissions. 
Deepseek-v4-pro has approximately 50 billion active parameters, so at the time of writing (early 2026), we are recommending $n=5\times10^{10}$ for frontier models.

Once Equation \ref{eq:simple_flops} has been used to obtain the number of FLOPs, the energy usage can be estimated using publicly released energy/T-FLOP data released by NVIDIA \cite{H100,McGinness2025Can}. 
Then once the energy usage is known, publicly available information about the average CO$_2$ emitted per kilowatt hour in the US (where most AI inference centres are currently located) can be used to estimate the CO$_2$ emitted\cite{CarbonIntensityGovernment,CarbonIntensityGlobal}. This process is summarised by the conversion factor in Equation \ref{eq:FLOP_CO2_conversion} which uses numbers which are reasonable for 2026.

\begin{flalign}
    \label{eq:FLOP_CO2_conversion}
    \text{kg of CO$_2$ emitted} &= n_{\text{FLOP }} \div \frac{10^{-12} \times \text{T-FLOPs}}{\text{KWh}} \times \frac{\text{CO}_2}{\text{KWh}} \nonumber \\ 
    &= n_{\text{FLOP }} \div (6.77\times10^{18} \text{FLOPs/KWh}) \times 0.384 \text{kg/KWh} \nonumber \\
    &= 5.67\times 10^{-20} \times  n_{\text{FLOP }}
\end{flalign}

If we use the assumption that $n=5\times10^{10}$ active parameters and assume that the reported number is the number of input and output tokens combined, we can answer the question posed in the opening paragraph of the introduction: \textit{How much CO$_2$ is released from 210 billion tokens?} Approximately $1200$kg of CO$_2$.

\section{Recommendations, Limitations and Future Work}

We believe that big tech companies should be transparent about the carbon footprint that LLMs (and other foundation models) have, or at very least provide information (such as number of active parameters, $n$) that can be used by researchers to estimate emissions from their use.  
Doing so will empower researchers to make decisions based on sustainability as well as model performance.  
In this vein, a significant focus of this paper is in encouraging researchers to proudly claim Smallest Model that Achieves the Job (SMAJ) as well as (SotA) in their future work. 

A key motivation for the theoretical work was to
improve on previous accuracy/carbon metrics which would reward under-trained models by more accurate modeling of the learning process. 
By separating $A_{\text{max}}$ and $k$ as two separate metrics (one for accuracy and one for training efficiency), we are able to accurately fit learning curves and produce a parameter which corresponds to the sustainability of training.
An area for future work would be to explore setting $A_{\text{max}}$ to the performance of the highest performing model and calculating $k$ values for other models based on this. This would allow $k$ to more effectively reflect carbon per accuracy while accounting for the diminishing returns in iterative training.

One limitation of our study is that we have only demonstrated \texttt{carbonbenchmark} on three relatively lightweight benchmarks with six well-established models.
At this stage it is unclear how well the exponential and power law fitting will apply to more computationally expensive models' learning curves.  
This requires further investigation.  
Furthermore, \texttt{codecarbon} only considers carbon dioxide emissions, focusing on a problem of global significance. However, other measures such as water consumption and amount of energy drawn from the local power grid are also having significant local societal impacts and are worthy of consideration. In future, we hope to extend \texttt{carbonbenchmark} to also include these metrics.

We believe that the reporting of CO$_2$ is a vital first step, however, we acknowledge that using it as a metric to determine the efficiency of research work may unfairly penalise people who live in countries with less-clean power grids or who only have access to less efficient CPUs, GPUs or RAM which require more energy per operation. 
Previous works have called on FLOPs to be reported as a way of leveling the playing field \cite{Eimler2026Environmental}, and while this works for GPU intensive workloads, it is not able to capture the full cost of all ML methods. A final area of future work would be to determine a metric which is both equitable and captures the computational cost of using CPU, GPU and RAM.

\begin{ack}
Use unnumbered first level headings for the acknowledgments. All acknowledgments
go at the end of the paper before the list of references. Moreover, you are required to declare
funding (financial activities supporting the submitted work) and competing interests (related financial activities outside the submitted work).
More information about this disclosure can be found at: \url{https://neurips.cc/Conferences/2026/PaperInformation/FundingDisclosure}.

Do {\bf not} include this section in the anonymized submission, only in the final paper. You can use the \texttt{ack} environment provided in the style file to automatically hide this section in the anonymized submission. 
\end{ack}

\bibliographystyle{plain} 
\small
\bibliography{PhDThesisReferences}

\normalsize

\appendix

\section{Appendix - Proofs of Theorems}
\label{Appendix:proof_of_theorems}

\begin{theorem}
    No general closed-form, analytic solution for training efficiency parameter $k$ can be determined from known values of $A$ and $t$, if we assume the learning curve takes the form $ A=A_{\text{max}}(1-e^{-kt})$ (Equation \ref{eq:exponential}). 
\end{theorem}

\begin{proof}
    As this equation has two parameters, exactly two data points $\{(t_1, A_1), (t_2, A_2)\}$ are required to solve for both $k$ and $A_{\text{max}}$. Substituting into Equation \ref{eq:exponential} we obtain a series of two simultaneous equations:
    \begin{equation*}
        A_1 = A_{\text{max}}(1-e^{-kt_1})
    \end{equation*}
    \begin{equation*}
        A_2 = A_{\text{max}}(1-e^{-kt_2})
    \end{equation*}
    Dividing these two equations gives $\frac{A_1}{A_2} = \frac{(1-e^{-kt_1})}{(1-e^{-kt_2})}$ which can be rearranged to the form 
    \begin{equation*}
        A_1 e^{-kt_2} - A_2 e^{-kt_1} = A_1 - A_2
    \end{equation*}
    By selecting substitution $x=e^{-k}$, this can be expressed as:
    \begin{equation*}
        A_1 x^{t_2} - A_2 x^{t_1} = A_1 - A_2
    \end{equation*}
    If the ratio of $t_1$ to $t_2$ is rational, we could choose an $n$ s.t. $t_1 = nt_2$ and then substitute $y=x^{t_1}$ to reduce the equation to 
    \begin{equation*}
        A_1 y - A_2 y^{n} = A_1 - A_2
    \end{equation*}
    Which according to the Abel-Ruffini impossibility theorem only has a solution if $n<5$. Therefore there is no general solution.
\end{proof}

We note that if the ratio of $t_1$ to $t_2$ is not rational (likely) then $A_1 x^{t_2} - A_2 x^{t_1} = A_1 - A_2$ is strictly transcendental and there is no way to invert the sum of two incommensurate exponentials.

\begin{theorem}
    No general closed-form, analytic solution for training efficiency parameter $\beta$ can be determined from known values of $A$ and $t$, if we assume the learning curve takes the form $ A = A_{\text{max}}(1 - \alpha t^{-\beta})$ (Equation \ref{eq:powerlaw}). 
\end{theorem}

\begin{proof}
    As this equation has three parameters, exactly three data points $\{(t_1, A_1), (t_2, A_2), (t_3, A_3)\}$ are required to solve for $\alpha$, $\beta$, and $A_{\text{max}}$. Substituting into Equation \ref{eq:exponential} we obtain a series of three simultaneous equations:
    \begin{equation*}
        A_1=A_{\text{max}}(1 - \alpha t_1^{-\beta})
    \end{equation*}
    \begin{equation*}
        A_2=A_{\text{max}}(1 - \alpha t_2^{-\beta})
    \end{equation*}
    \begin{equation*}
        A_3=A_{\text{max}}(1 - \alpha t_3^{-\beta})
    \end{equation*}
    Pairwise subtraction of these equations gives:
    \begin{equation*}
        A_3-A_1=A_{\text{max}}\alpha (t_1^{-\beta} - t_3^{-\beta})
    \end{equation*}
    \begin{equation*}
        A_3-A_2=A_{\text{max}}\alpha (t_2^{-\beta} - t_3^{-\beta})
    \end{equation*}
    Then division of these equations gives:
    \begin{equation*}
        \frac{A_3-A_1}{A_2-A_1} = \frac{t_1^{-\beta} - t_3^{-\beta}}{t_2^{-\beta} - t_3^{-\beta}}
    \end{equation*}
    For simplicity let $C=\frac{A_3-A_1}{A_2-A_1}$, then we can rewrite this as:
    \begin{equation*}
        C t_2^{-\beta} - C t_3^{-\beta} = t_1^{-\beta} - t_3^{-\beta}
    \end{equation*}
    Which can be rewritten as:
    \begin{equation*}
        (1-C)e^{-\beta \ln(t_3)} + Ce^{-\beta \ln(t_2)} - e^{-\beta \ln(t_1)} = 0
    \end{equation*}
    Which by the arguments made in the proof of Theorem \ref{theorem:exponential} has no general solution according to the Abel-Ruffini impossibility theorem.
\end{proof}

We also note that it is highly unlikely that three data points chosen for fitting for $\alpha$, $\beta$ and $A_{\text{max}}$ are unlikely to fall exactly on a $A = A_{\text{max}}(1 - \alpha t^{-\beta})$ curve as this curve cannot be reshaped to match any three arbitrary points.

\section{Appendix - Least Square Fitting with Uncertainty}
\label{appendix:least_squares}
We aim to model the learning curve of an ML model. 
We the task is to fit set of data $=\{(t_1, A_1), (t_2, A_2), ..., (t_m, A_m)\}$ to the learning curve function $f(t_i,\theta)$. 
$t_i$ are the computional costs at each point in the training, $A_i$  are the corresponding accuracies ans $\theta$ are the parameters to fit the learning curve model.
For the exponential fit $\theta = (A_{\text{max}},k)$, for the power-law model $\theta = (A_{\text{max}},\alpha,\beta)$, and for delayed exponential fit $\theta = (A_{\text{max}},k,t_0)$.

The predicted accuracy for a data point is given by:
\begin{equation*}
    A_i' = f(t_i, \theta)
\end{equation*}
Where $f()$ is the learning curve fitting function given by either Equation \ref{eq:exponential}, \ref{eq:delayed_exponential}, or \ref{eq:powerlaw}. 
For example for the delayed exponential fitting function, the predicted value $A_i'$ is given by:
\begin{equation*}
    A_i' = A_{\text{max}}(1-e^{-k(t_i-t_0)})
\end{equation*}
The objective of fitting is to find the optimal parameters $\hat{\theta}$ that minimises the sum of square residuals denoted $R^2(\theta)$ given by:
\begin{equation}
    R^2(\theta) = \sum_{i=1}^{m}(A_i - A_i')^2
\end{equation}
Because $f(t_i,\theta)$ is non-linear, we cannot solve for $\hat{\theta}$ analytically and instead use an iterative, numerical method which is implemented by \texttt{scipy.optimize.curve\_fit} as follows.
At each iteration, the algorithm linearises $f(t_i,\theta)$ around $\theta$ using the $m \times p$ Jacobian matrix $\mathbf{J}$, where $p$ is the number of parameters. 
Each entry of the Jacobian is defined as follows:
\begin{equation}
    \mathbf{J} = \frac{\partial f(t_i,\theta)}{\partial \theta_j}
\end{equation}
$\mathbf{J}$ determines the direction and magnitude of a parameter update required to reach the minimum of $R^2(\theta)$, thus obtaining $\hat{\theta}$ by gradient descent.

Once $\hat{\theta}$ has been reached, $\mathbf{J}$ is used to quantify the uncertainty in the learned parameters, by producing a co-variance matrix, $\mathbf{C}$ as follows:
\begin{equation*}
    \mathbf{C} = \frac{R^2(\hat{\theta})}{n-p}(\mathbf{J}^T\mathbf{J})^{-1}
\end{equation*}
The uncertainty (standard error) for each parameter is obtained using the diagonal elements of the covariance matrix, $C_{jj}$, using the formula:
\begin{equation}
    \label{eq:uncertainty}
    \Delta \theta_j = \sqrt{C_{jj}}
\end{equation}
We use the values produced from Equation \ref{eq:uncertainty} as estimates of the uncertainty in the learned parameters for model fitting.

\section{Appendix - Additional Experimental Results for Learning Curve Fitting}
\label{appendix:additional_results}


\begin{figure}[htbp]
    \centering
    \includegraphics[width=\textwidth]{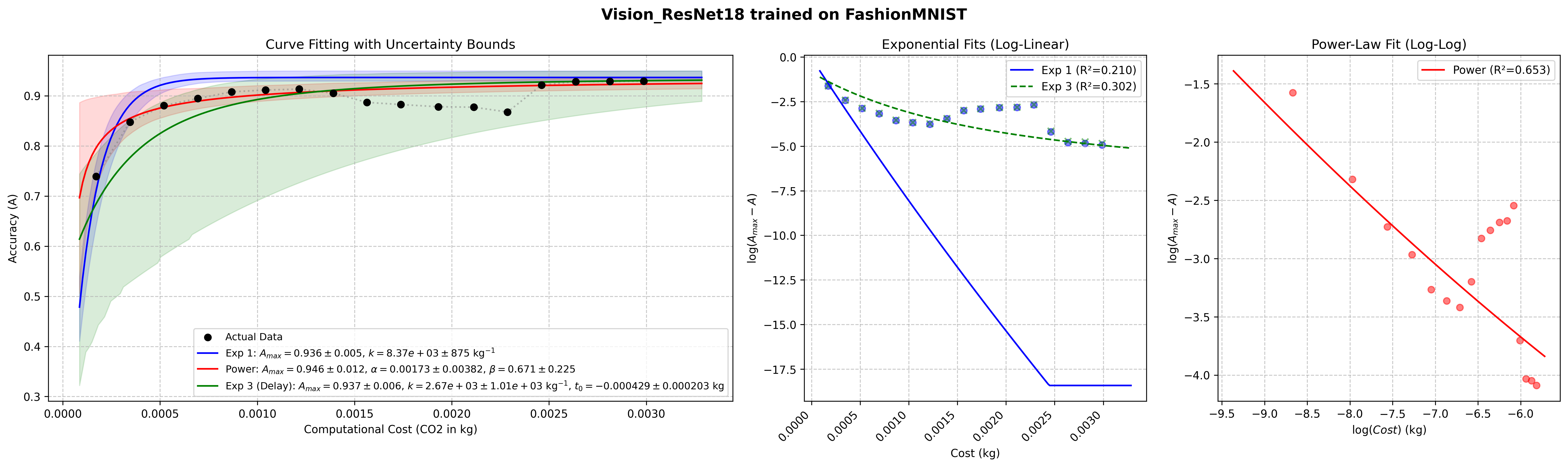}
    \caption{Learning curve and mathematical fits evaluating the computational efficiency of ResNet18 trained on the FashionMNIST dataset.}
    \label{fig:vision_resnet18}
\end{figure}


\begin{figure}[htbp]
    \centering
    \includegraphics[width=\textwidth]{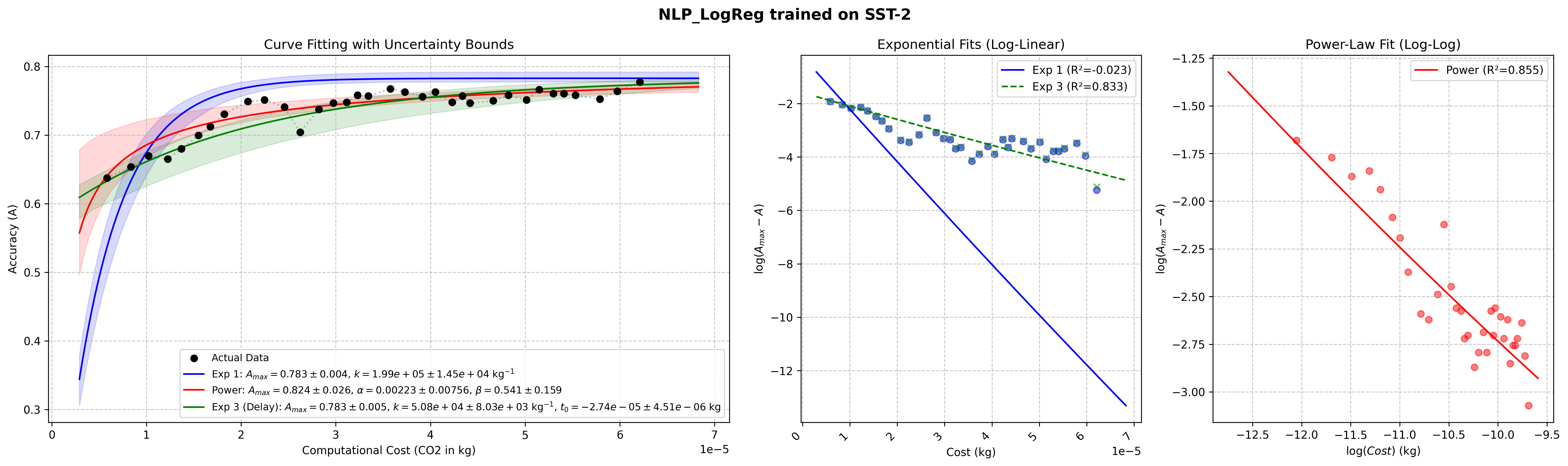}
    \caption{Learning curve and mathematical fits evaluating the computational efficiency of Logistic Regression on the SST-2 text classification dataset.}
    \label{fig:nlp_logreg}
\end{figure}

\begin{figure}[htbp]
    \centering
    \includegraphics[width=\textwidth]{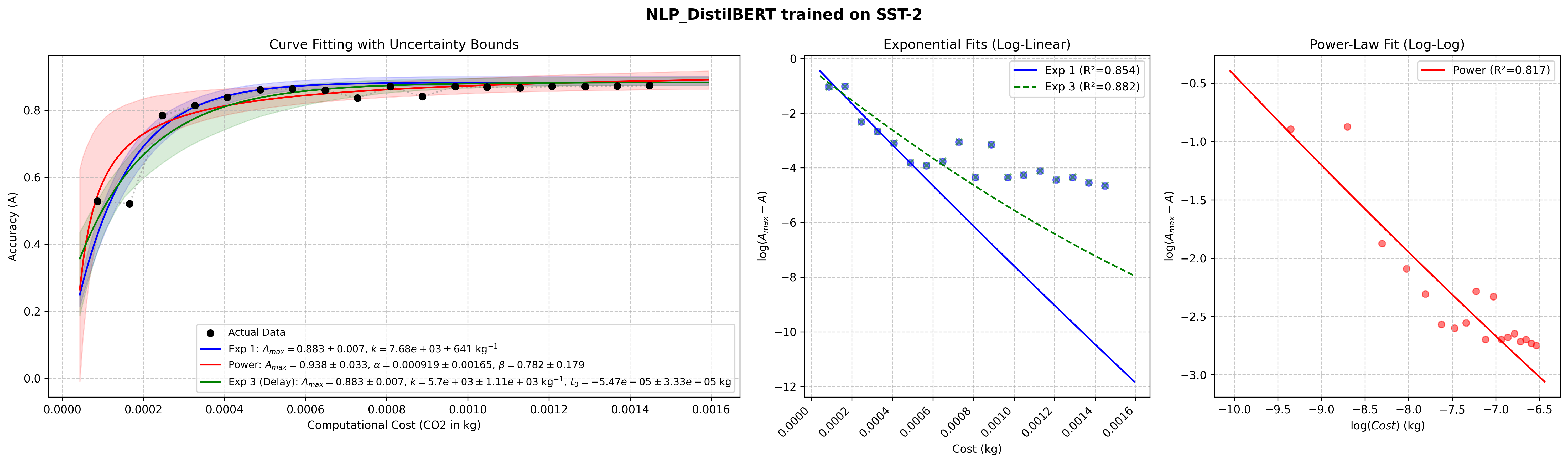}
    \caption{Learning curve and mathematical fits evaluating the computational efficiency of DistilBERT fine-tuned on the SST-2 dataset.}
    \label{fig:nlp_distilbert}
\end{figure}


\begin{figure}[htbp]
    \centering
    \includegraphics[width=\textwidth]{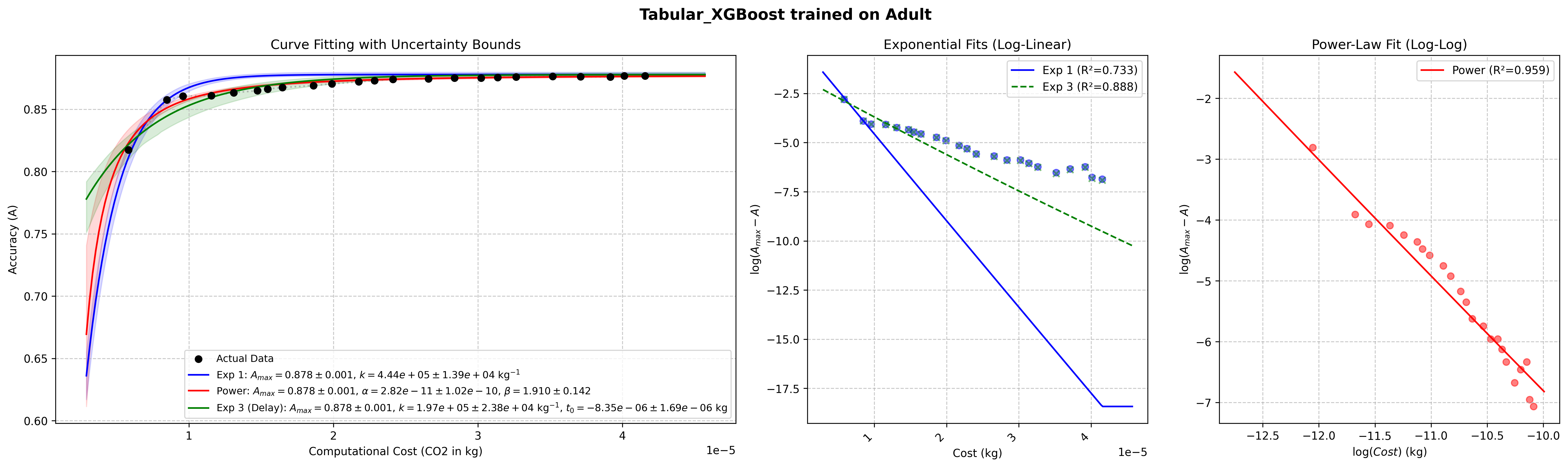}
    \caption{Learning curve and mathematical fits evaluating the computational efficiency of XGBoost trained on the Adult tabular dataset.}
    \label{fig:tabular_xgboost}
\end{figure}

\begin{figure}[htbp]
    \centering
    \includegraphics[width=\textwidth]{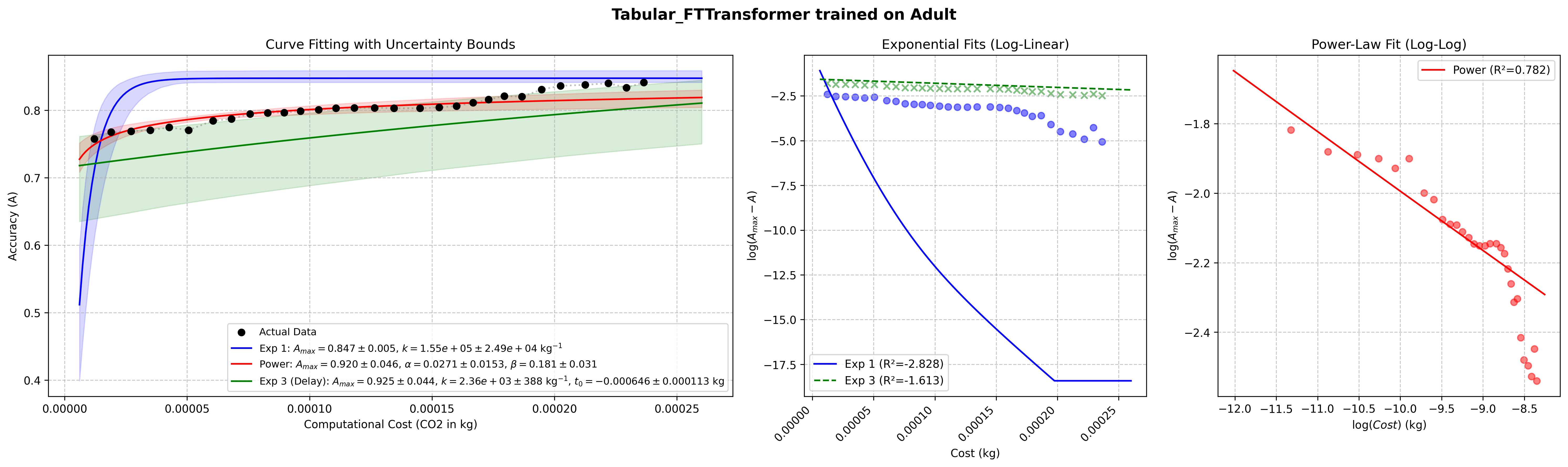}
    \caption{Learning curve and mathematical fits evaluating the computational efficiency of FT-Transformer trained on the Adult tabular dataset.}
    \label{fig:tabular_fttransformer}
\end{figure}

\FloatBarrier

\section{Appendix - Literature Review Results and Prompt}
\label{appendix:literature_review}

When comparing models, in addition to randomly selecting 200 papers, we also systematically chose the first 200 papers (by author name). We found that there was no significant difference between the two samples. 

The full prompt given to the models to perform the review is given below:

{\tt 
\begin{quote}
\small
\begin{verbatim}
Please read the attached academic article and then follow the exact
instructions below to answer questions about it.

1.) Write "Start:"

2.) Then write "Paper uses LLMs: "

3.) Then write either "Yes" or "No"

4.) Then write "Type of LLM Usage: "

5.) Then write at least one of the following separated by commas if
you choose more than one:

- "Architecture/Pre-training" if the paper proposes a new model
- "Fine-tuning" if the paper adjusts an existing model for a specific task
- "Inference/Evaluation" if the paper uses the LLM as a judge for
  labelling data
- "Synthetic Data Generation" if the paper uses LLMs to create
  training data
- "Agentic Frameworks" if the LLM is used for reasoning, solving a
  problem or completing a task

You may select more than one of these options.

If your answer to task 3 was "No" then instead write "None" for
this task.

6.) Write "Open Source: "

7.) Then write "Yes" if the paper uses only open source LLMs where
the number of parameters are known. Write "No" if the paper does not
use LLMs or uses any proprietary models.

8.) Then write "Training an ML Model on a standard benchmark: "

9.) Then write "Yes" if the paper trains a machine learning model on
a standard benchmark.

10.) Write "Telemetry Available: "

11.) Then write "Yes" if the paper says the authors ran models on
hardware or infrastructure they directly controlled or rented at the
VM/instance level and could therefore estimate their compute.
Write "No" if they used a managed platform or API that abstracts the
hardware and reports only high-level usage.
Write "Unclear" if this cannot be determined from the paper and "NA"
if the paper did not require any compute.

12.) Write "Computational Cost: "

13.) Then write "Yes" if the paper reports the computational cost of
its methods (even if the paper is not about LLMs).
Write "No" if computational cost is not mentioned in the paper.
Write "NA" if the paper did not perform any experiments or use any
compute.

14.) Then write "Sustainability: "

15.) Then write "Yes" if the carbon footprint or other environmental
impacts were considered as part of the paper or "No" if these were
not mentioned in the paper.

16.) Write ":End:"
\end{verbatim}
\end{quote}
}

\section{Appendix - Statistical Comparison of LLMs in Literature Review Task}\label{section:appendix_statistics}

To identify whether individual LLMs differed significantly from a SotA baseline model (here chosen to be Gemini 3.1 Pro Preview), we used a Generalized Linear Mixed Model (GLMM).  For a set of 200 papers, each LLM was asked to answer a series of questions about each paper where each of the questions had a yes/no answer that could be modeled as a Bernoulli response variable.  The GLMM was formulated as:

\begin{align*}
Y_{ijk} &\sim \textup{Bernoulli}(p_ijk) \\
\textup{logit}(p_{ijk}) &= \eta_{ijk} \\
\eta_{ijk} &= \alpha + \beta_i + \gamma_j + r_{k} \\
\end{align*}

\noindent where $Y_{ijk}$ denotes the response of the $i$th LLM to the $j$th question, for the $k$th paper; $\textup{logit}(p) := \frac{p}{1 - p}$; $\alpha$ is the intercept parameter of the linear model; $\beta_i$ is the fixed effect of the $i$th LLM on the response variable outcome; $\gamma_j$ is the fixed effect of the $j$th question on the response outcome; and $r_{k} \sim N (0, \sigma_r^2)$ is a random effect that describes the effect of the $k$th paper on the response outcome.

The GLMM was fitted to the experimental data using the packages {\tt glmmTMB} \cite{brooks2017glmmtmb} in the {\tt R} statistical programming language \cite{Rcore2026}.  Since we wished to compare each LLM to a SotA benchmark model, we used Dunnett's test to do multiple comparisons.  In this case, assessing whether each LLM's effect in the GLMM was significantly different to that of Gemini 3.1 Pro Preview.  Dunnett's test was performed in {\tt R} using the {\tt emmeans} \cite{lenth2024emmeans} package.  The results of the multiple comparisons are presented in Table \ref{table:Dunnett}.  Based on the p-values in this table, we see no strong evidence to suggest that Haiku 4.5, Claude Opus 4.6, Claude Sonnet 4.6, or Gemini 3.1 Flash Lite are significantly different to the Gemini 3.1 Pro Preview model.  On these grounds we retain the null hypothesis that there is no evidence of a difference between these four models and Gemini 3.1 Pro Preview model.

\begin{table}
\caption{Pairwise comparisons of LLMs to the chosen baseline model (Gemini 3.1 Pro Preview) using Dunnett's test.}
\label{table:Dunnett}
\centering
\begin{tabular}[t]{lrrrl}
\toprule
contrast & odds.ratio & SE & z.ratio & p.value\\
\midrule
Haiku 4.5 / Gemini 3.1 Pro Preview & 1.083 & 0.130 & 0.6623202 & 0.9396\\
Claude Opus 4.6 / Gemini 3.1 Pro Preview & 0.801 & 0.099 & -1.7892710 & 0.3223\\
Claude Sonnet 4.6 / Gemini 3.1 Pro Preview & 1.068 & 0.129 & 0.5426289 & 0.967\\
Gemini 3.1 Flash Lite / Gemini 3.1 Pro Preview & 1.154 & 0.138 & 1.1973763 & 0.6921\\
Gemini 3 Flash Preview / Gemini 3.1 Pro Preview & 0.500 & 0.065 & -5.3050248 & <0.0001\\
\addlinespace
Gemma 3.12b / Gemini 3.1 Pro Preview & 1.624 & 0.190 & 4.1447230 & 2e-04\\
Gemma 3.27b / Gemini 3.1 Pro Preview & 2.047 & 0.237 & 6.1977542 & <0.0001\\
\bottomrule
\end{tabular}
\end{table}


\end{document}